\pdfoutput=1
\documentclass[11pt]{article}

\usepackage[compress]{natbib}
\usepackage{amsmath,amssymb,amsthm}
\usepackage{color}
\usepackage[margin=1in]{geometry}
\usepackage{setspace}
\usepackage{amsthm, amsmath, amssymb, amsfonts}
\usepackage{color}
\usepackage[dvipsnames]{xcolor}
\usepackage{hyperref}
\hypersetup{pdftex, colorlinks=true, citecolor=MidnightBlue, linkcolor=MidnightBlue, urlcolor=MidnightBlue}

\usepackage{url}            
\usepackage{booktabs}       
\usepackage{nicefrac}       
\usepackage{graphicx}
\usepackage{caption}
\usepackage{thm-restate}
\usepackage[linesnumbered, ruled]{algorithm2e}
\usepackage{algpseudocode}
\usepackage{bbm}
\usepackage[export]{adjustbox}

\usepackage{microtype}
\usepackage{mathpazo} 
\usepackage[scaled]{helvet} 
\usepackage{courier} 
\normalfont
\usepackage[T1]{fontenc}

\newcommand{\Stay}{\texttt{Stay}}
\newcommand{\Move}{\texttt{Move}}
\newcommand{\E}{\mathbb{E}}
\newcommand{\reals}{\mathbb{R}}
\newcommand{\Prob}{\mathbb{P}}

\newcommand{\Ac}{\mathcal{A}}
\newcommand{\Sc}{\mathcal{S}}

\newcommand{\Qc}{\mathcal{Q}}

\newcommand{\Xc}{\mathcal{X}}

\newcommand{\N}{\mathbb{N}}
\DeclareMathOperator*{\argmin}{argmin}
\DeclareMathOperator*{\argmax}{argmax}
\def\ind{\mathbbm{1}}

\newtheorem{thm}{Theorem}

\newtheorem{lemma}[thm]{Lemma}

\newtheorem{example}{Example}
\newtheorem{remark}{Remark}

\newtheorem{defn}{Definition}

\title{Approximation Benefits of Policy Gradient Methods \\with Aggregated States}

\date{}
\author{%
  Daniel J. Russo \\
  Columbia University\\
  New York, NY 10027 \\
  \texttt{djr2174@gsb.columbia.edu} \\
}

\begin{document}
\maketitle

\begin{abstract}
	Folklore suggests that policy gradient can be more robust to misspecification than its relative, approximate policy iteration. This paper studies the case of state-aggregated representations,  where the state space is partitioned and either the policy or value function approximation is held constant over partitions. This paper shows a policy gradient method converges to a policy whose regret per-period is bounded by $\epsilon$, the largest difference between two elements of the state-action value function belonging to a common partition. With the same representation, both approximate policy iteration and approximate value iteration can produce policies whose per-period regret scales as $\epsilon/(1-\gamma)$, where $\gamma$ is a discount factor. Faced with inherent approximation error, methods that locally optimize the true decision-objective can be far more robust.
	
\end{abstract}

\section{Introduction}
The fields of approximate dynamic programming and reinforcement learning offer algorithms that can produce effective performance in complex control problems where large state spaces render exact computation intractable. Two of the classic methods in this area, approximate value iteration (AVI) and approximate policy iteration (API), modify value and policy iteration by fitting parametric approximation to the value function. In recent years, an alternative class of algorithms known as policy gradient methods has surged in popularity. These algorithms search directly over a parameterized subclass of policies by applying variants of gradient ascent to an objective function measuring the total expected reward accrued.  An appealing feature of these methods is that, even though parametric approximations may introduce unavoidable error, they still directly optimize the true decision objective. AVI and API, by contrast, pick parameters by minimizing a measure of error in the value function approximation that may not be well aligned with the decision objective. 

Numerous papers have found empirically that direct policy search eventually converges to superior policies. For example, consider a long sequence of works that applied approximate dynamic programming techniques to Tetris\footnote{See \cite{gabillon2013approximate} for a full account of the history.}.   After \cite{bertsekas1996temporal}  applied approximate policy iteration to the problem, 
 \cite{kakade2002natural}, \cite{szita2006learning},  and \cite{furmston2012unifying} attained much higher scores using methods that directly search over a class of policies. This is not unique to tetris.  A similar phenomenon was observed in an ambulance redeployment problem by \cite{maxwell2013tuning} and a battery storage problem by \cite{scott2014least}. Experiments with deep reinforcement learning tend to be less transparent, but policy gradient methods are extremely popular \citep[see e.g.][]{schulman2015trust, schulman2017proximal}.


\cite{kakade2002natural, szita2006learning, furmston2012unifying, maxwell2013tuning} and \cite{scott2014least} all search over the class of policies that are induced by (soft) maximization with respect to some parameterized value function. In a sense, these methods tune the parameters of the value function approximation, but do so aiming to directly improve the total expected reward earned rather than to minimize a measure of prediction error. As a result, any gap in performance cannot be due to the approximation architecture and instead is caused by the procedure that sets the parameters. 




There is very limited theory formalizing this phenomenon. Several works provide broad performance guarantees for each type of algorithm. In the case of API, \cite{munos2003error, antos2008learning} and \cite{lazaric2012finite} build on the original analysis of \cite{bertsekas1996neuro}. An intellectual foundation for studying policy gradient methods was laid by \cite{kakade2002approximately}, who analyze a conservative policy iteration algorithm (CPI).  \cite{scherrer2014local} observed that guarantees similar to those for CPI could be provided for some idealized policy gradient methods and recently  \cite{agarwal2019optimality} developed approximation guarantees and convergence rates for a much broader class of policy gradient algorithms. There are few lower bounds, but comparing available upper bounds is suggestive. The results for incremental algorithms like CPI depend on a certain distribution shift term that is typically smaller than the so-called concentrability coefficients in \cite{munos2003error, antos2008learning, lazaric2012finite}. See \cite{scherrer2014approximate}. 

This paper provides a specialized study of algorithms that use state-aggregated representations, under which the state space is partitioned and either the policy or value function approximation does not distinguish between states in a common partition. State aggregation is a very old idea in approximate dynamic programming and reinforcement learning \citep{whitt1978approximations, bean1987aggregation, singh1995reinforcement, gordon1995stable, tsitsiklis1996feature,rust1997using, li2006towards, jiang2015abstraction, abel2016near}, leading to tractable algorithms for problems with low dimensional continuous state spaces where it is believed that nearby states are similar. We measure the inherent error of a state aggregation procedure by the largest difference between two elements of the state-action value function belonging to a common partition, denoted $\epsilon_{\phi}$. (Here $\phi$ denotes a particular state-aggregation.) 

We show that any policy that is a stationary point of the policy gradient objective function has per-period regret less than $\epsilon_{\phi}$. Many variants of policy gradient algorithms, being first-order methods, are ensured to converge (often efficiently) to stationary points, so this provides a guarantee on the quality of an ultimate policy produced with this approximation architecture. This guarantee is a substantial improvement over past work. The recent results of \cite{bhandari2019global} translate into limiting per-period regret of $\kappa_{\rho} \epsilon_{\phi} / (1-\gamma)$, where $\gamma$ is a discount factor and $\kappa_{\rho}$ is a complex term that captures distribution shift. Critically, here even per-period regret scales with the effective horizon. Other available bounds \citep{kakade2002approximately, scherrer2014local, agarwal2019optimality} are at least as bad\footnote{See Remark \ref{rem:updates_to_agrawal} for discussion of updated results in \cite{agarwal2019optimality}.}. Building on an example of \cite{bertsekas1996neuro}, we give an example in which API produces policies whose per-period regret scales as $\epsilon_{\phi} / (1-\gamma)$, hence establishing formally that policy gradient methods converge to a far better policy with the same approximation architecture. 

The large performance gap between API and policy gradient is surprising since they are known to be closely related \citep{konda2000actor, sutton2000policy}. There is a particularly precise and simple connection in the case of state-aggregation: Theorem \ref{thm:pg-is-api} shows that a Frank-Wolfe \citep{frank1956algorithm} policy gradient method is equivalent to a version of a API which (1) estimates an approximate value function by minimizing a loss function which weighs errors at states in proportion to how often those states are visited under the current policy and (2) makes soft or local updates to the policy in each iteration. With these modifications, each iteration of API locally optimizes a first-order approximation to the true decision objective, leading to much greater robustness to approximation errors. See Section \ref{sec:actor_critic}. Section \ref{sec: api} also studies a version of API that makes only the first change --- using an on-policy state weighting --- but not the second. A different counterexample (Ex.~\ref{counterexample2}) is constructed to show that this variant can be as brittle as a standard version of API which makes neither change.

The limited scope of this work should be highlighted. I have in mind settings where policy gradient is applied in simulation, which is how they are currently employed in nearly all applications. In this case, it is feasible to employ an exploratory initial distribution as is assumed in this work. This choice can have a critical impact on the optimization landscape \cite[see e.g.][]{agarwal2019optimality}. By focusing on the quality of stationary points, the paper sidesteps issues of (i) optimization error (i.e. error due to executing only a finite number of gradient steps), (ii) statistical error (i.e. error due to estimating gradients with limited simulated rollouts) and (iii) precisely which policy gradient variant and stepsizes are employed. The treatment of API is similarly stylized. This choice allows for a crisp presentation focused on one insight that is missing in the current literature.

\paragraph{Further discussion of related literature.} 
\cite{dong2019provably} prove that a state-aggregated and optimistic variant of $Q$-learning efficiently approaches limiting per-period regret smaller a measure of inherent aggregation error similar to $\epsilon_{\phi}$. \cite{van2006performance} had previously shown that approximate value iteration with fixed state-relevance weights could suffer per-period regret that scales with the effective time horizon. \cite{van2006performance} also observed that the robustness of policies derived from solutions of state-aggregated Bellman equations can depend critically on the choice of state-relevance weights. While this paper studies different algorithms and gives proofs which bear little resemblance to \cite{van2006performance} and \cite{dong2019provably},  their study of the robustness difference between to closely related algorithms  inspired my own. 

A number of recent works study the convergence rates of policy gradient methods in Markov decision processes with finite state and action spaces under the assumption that the policy class can represent all stochastic policies. Examples include \citep{agarwal2019optimality, mei2020global, zhang2021sample, cen2021fast, bhandari2021linear, khodadadian2021linear}. The fastest convergence rates seem to be attained by policy gradient variants that behave just like policy iteration \citep{bhandari2021linear}. Such theory therefore does not offer insight into the advantages of policy gradient methods over classical algorithms. This paper helps closes that. intellectual gap in the literature.

A fascinating paper by \cite{scherrer2012use} observed that the horizon dependence of approximate value and policy iteration can be improved by modifying them to use non-stationary policies. It is unclear if there is a connection between their work and this paper's finding about policy gradient methods for optimizing over the class of stationary policies. 

What is called state-aggregation in this paper is sometimes called `hard' state aggregation. Generalizations allow a state to have an affinity or partial association with several different regions of the state space. \cite{singh1995reinforcement} propose a soft state aggregation method, which avoids discontinuities in the approximate value function at the boundaries of partitions. The theory of \cite{tsitsiklis1996feature} holds for a related class of approximations they call interpolators. A popular approximation method called tile coding, which is closely related to state-aggregation, is discussed in the textbook of \cite{sutton2018reinforcement}. This paper does not pursue such extensions, focusing on the simplest variant of state-aggregation. The current proof significantly relies on the structure of hard state-aggregation.

 The current paper assumes a state-aggregation rule is fixed and given, then studying the quality of the policies various algorithms produce with this approximation.  A long series of works have focused on how a state aggregation can be learned or adaptively constructed  \citep{bertsekas1989adaptive, singh1995reinforcement, dean1997model, jiang2015abstraction, duan2019state, misra2020kinematic}. 
  
The bounds for policy gradient provided in this work depend on a notion of maximal error due to state aggregation (See Definition \ref{def:error}). Contemporaneously, \cite{agarwal2020pc} looked at policy gradient with state aggregated representations. Dependence on the time horizon is not a focus of their work. Instead, they focus on providing an upper bound that depends on a softer notion of approximation error which averages across partitions. It seems that the two measures coincide in Example \ref{counterexample} below, but the measure of \cite{agarwal2020pc} can offer an important improvement in other examples. Their upper bounds are not matched by those previously established for alternative algorithms and are suggestive of another provable benefit of using policy gradient methods in this setting.

\section{Problem Formulation}\label{sec: formulation}
 We consider a Markov decision process $M = (\Sc, \Ac, r, P, \gamma, \rho)$, which consists of a finite state space $\Sc=\{1,\cdots, |\Sc|\}$, finite action space $\Ac=\{1, \cdots, |\Ac| \}$, reward function $r$, transition kernel $P$, discount factor $\gamma \in (0,1)$ and initial distribution $\rho$. For any finite set $\Xc=\{1,\cdots, |\Xc|\}$, we let $\Delta(\Xc) = \{ d\in \reals^{|\Xc|}_{+} : \sum_{x\in \Xc} d_x =1  \}$ denote the set of probability distributions over $\Xc$. A stationary randomized policy is a mapping  $\pi: \Sc \to \Delta(\Ac)$. We use $\pi(s,a)$ to denote the $a^{\text{th} }$ component of $\pi(s)$. Let $\Pi$ denote the set of all stationary randomized policies. Conditioned on the history up to that point, an agent who selects action $a$ in state $s\in \Sc$ earns a reward in that period with expected value $r(s,a)$ and transitions randomly to a next state, where $P(s' | s, a)$ denotes the probability of transitioning to state $s'\in \Sc$. 
To treat randomized policies, we overload notation, defining for $d \in \Delta(\Ac)$, $r(s,d)= \sum_{a=1}^{|\Ac|} r(s,a)d_a$ and $P(s'|s, d)=\sum_{a=1}^{|\Ac|}  P(s'| s, a)d_a$. Notice that if $e_a\in \Delta(\Ac)$ is the $a^{\text{th}}$ standard basis vector, then $r(s,e_a)=r(s,a)$. 

\paragraph{Value functions and Bellman operators.}
We define, respectively, the value function associated with a policy $\pi$ and the optimal value function by 
\[ 
V_{\pi}(s) = \E_s^{\pi}\left[  \sum_{t=0}^{\infty} \gamma^t r(s_t, a_t ) \right], \qquad V^*(s) = \max_{\pi \in \Pi}  V_{\pi}(s).
\]
The notation $\E_s^{\pi}[\cdot]$ denotes expectations taken over the sequence of states when $s_0=s$ and policy $\pi$ is applied. A policy $\pi^*$ is said to be optimal if $V_{\pi^*}(s) = V^*(s)$ for every $s\in \Sc$. It is known that  an optimal deterministic policy exists. Throughout this paper, I will use $\pi^*$ to denote some optimal policy. There could be multiple, but this does not change the results. The Bellman operator $T_{\pi} : \reals^n \to \reals^n$ associated with a policy $\pi \in \Pi$ maps a value function $V\in \mathbb{R}^n$ to a new value function $T_{\pi} V \in \reals^n$ defined by 
$\left(T_{\pi} V\right)(s) = r(s,\pi(s)) + \gamma \sum_{s' \in \Sc} P(s' | s, \pi(s)) V(s')$.
The Bellman optimality operator $T: \reals^n \to \reals^n$ is defined by 
\[
TV(s) = \max_{\pi \in \Pi} \left(T_{\pi} V\right) (s) = \max_{d \in \Delta(\Ac) } r(s,d) + \gamma \sum_{s' \in \Sc} P(s' | s, d) V(s').
\]
It is well known that $T$ and $T_{\pi}$ are contraction mappings with respect the maximum norm. Their unique fixed points are $V^*$ and $V_{\pi}$, respectively. For a state $s\in \Sc$, policy $\pi \in \Pi$ and action distribution $d\in \Delta(\Ac)$, define the state-action value function $Q_{\pi}(s,d) = r(s,d) + \gamma \sum_{s' \in \Sc} P(s' | s, d) V_{\pi}(s),$ which measures the expected total discounted reward of sampling an action from $d$ in state $s$ and applying $\pi$ thereafter. When $d$ is deterministic, meaning $d_a = 1$ for some $a\in \Ac$, we denote this simply by $Q_{\pi}(s,a)$. Define $Q^*(s,d)=Q_{\pi^*}(s,d)$ for some optimal policy $\pi^*$. These obey the relations, 
\begin{equation}\label{eq: Q to bellman}
Q_{\pi}(s, \pi'(s) ) = \left( T_{\pi'} V_{\pi}\right)(s)  \qquad    \max_{d\in \Delta(\Ac) }  Q_{\pi}(s, d ) = \left( T V_{\pi} \right)(s).
\end{equation}

\paragraph{Geometric average rewards and occupancies.}
Policy gradient methods are first order optimization algorithms applied to optimize a scalar objective that measures expected discounted reward earned from a random initial state, given by
\begin{equation}\label{eq:objective_J}
J(\pi) = (1-\gamma) \sum_{s \in \Sc} \rho(s) V_{\pi}(s).
\end{equation}
Another critical object is the discounted state occupancy measure 
\[ 
\eta_{\pi} = (1-\gamma)\sum_{t=0}^{\infty}  \gamma^t \rho  P_{\pi}^t \in \Delta(\Sc),
\]
where $P_{\pi}\in \reals^{|\Sc|\times |\Sc|}$ is the Markov transition matrix under policy $\pi$ and $\rho \in \mathbb{R}^{|\Sc|}$ is viewed as a row vector. Here $\eta_{\pi}(s)$ gives the geometric average time spent in state $s$ when the initial state is drawn from $\rho$. These two are related, as $J(\pi) = \sum_{s \in \Sc} \eta_{\pi}(s) r(s, \pi(s))$. 

The factor of $(1-\gamma)$ in the definitions of $\eta_{\pi}$ and $J(\pi)$ serves to normalize these quantities and gives them a natural interpretation in terms of average reward problems. In particular, consider, just for the moment, a problem with modified transition probabilities $\tilde{P}(s'| s, a) = (1-\gamma)\rho(s')+ \gamma P(s' | s, a)$. That is, in each period there is a $1-\gamma$ chance that the system resets in a random state drawn from $\rho$. Otherwise, the problem continues with next state drawn according to $P$. One can show that  $J(\pi)$ denotes the average reward earned by $\pi$ and $\eta_{\pi}(s)$ is average fraction of time spent in state $s$ under policy $\pi$ in this episodic problem. Undiscounted average reward problems are often constructed by studying $J(\pi)$ as the discount factor approaches one \citep{bertsekas1995dynamic, puterman2014}.

\section{State aggregation }\label{sec: state agg}

A state aggregation is defined by a function $\phi: \Sc \to \{1, \cdots, m \}$ that partitions the state space into $m$ segments. We call $\phi^{-1}(j) = \{ s \in \Sc : \phi(s) = j\}$ the $j$--th segment of the partition. Typically we have in mind problems where the state space is enormous (effectively infinite) but it is tractable to store and loop over vectors of length $m$. Tractable algorithms can then be derived by searching over approximate transition kernels, value functions, or policies, that don't distinguish between states belonging to a common segment. Our hope is that states in a common segment are sufficiently similar, for example due to smoothness in the transition dynamics and rewards, so that these approximations still allow for effective decision-making.

To make this idea formal, let us define the set of approximate value functions and policies induced by a state aggregation $\phi$,
\begin{align*}
\Qc_{\phi}&= \{ Q\in \reals^{|\Sc| \times |\Ac|  }  :   Q(s,a) = Q(s', a) \,\, \text{for all } a\in \Ac, \text{ and all } s,s'\in\Sc \,\, \text{such that }  \phi(s)=\phi(s') \}\\
\Pi_{\phi} &= \{ \pi \in \Pi :  \pi(s) = \pi(s')  \, \text{for all } s,s'\in\Sc \,\, \text{such that }  \phi(s)=\phi(s') \}. 
\end{align*}
It should be emphasized that practical algorithms do not require, for example, actually storing $n\cdot |\Ac|$ numbers in order to represent an element $Q\in \Qc_{\phi} \subset \reals^{|\Sc| \cdot |\Ac|}$. Instead, one stores just $m\cdot |\Ac|$ numbers, one per segment. 

Should we approximate the value function or the policy? In this setting, there is a broad equivalence. This is important to the interpretation of the paper's results, since it implies that any benefit of policy gradient methods is not due to its approximation architecture but instead due to how it searches over the parameters of an approximation. A complete proof of these statements is omitted but some details are given Appendix \ref{sec:policy_value_equivalence}.
\begin{remark}[Equivalence of aggregated-state approximations]\label{rem:equivalence_of_representations}
	The set of randomized state-aggregated policies $\Pi_{\phi}$ is equal to the set of policies formed by softmax optimization with respect to state-aggregated value functions: 
	\begin{equation}\label{eq:state-agg_policies_and_softmax}
	\Pi_{\phi} = {\rm closure}\{ \pi \in \Pi : Q\in \Qc_{\phi},\,\, \pi(s,a)  = e^{Q(s, a)}/ \sum_{a'\in \Ac} e^{Q(s, a')} \,\, \forall s\in \Sc, a\in \Ac   \}. 
	\end{equation}
	
	Moreover, the set of deterministic policies contained in $\Pi_{\phi}$ is isomorphic to 
	\begin{equation}\label{eq:deterministic_policies}
	\{ f\in \Ac^{|\Sc|} : Q \in \Qc_{\phi} ,  f(s) = \min\{\argmax_{a\in \Ac} Q(s,a)  \} \,  \},
	\end{equation}
	the set of greedy policies\footnote{	Here we have broken ties deterministically in favor of the actions with a smaller index. If there are multiple optimal actions and ties are broken differently at states sharing common segment, the induced policy would not be constant across segments. } with respect to some state-aggregated value function. 
\end{remark}


\section{Approximation via policy gradient with state-aggregated policy classes}\label{sec:upper_bounds}
\paragraph{Convergence to stationary points.}

Policy gradient methods are first-order optimization methods applied to maximize $J(\pi)$ over the constrained policy class $\Pi_{\phi}$. Of course, just as there is an ever growing list of first-order optimization procedures, there are many policy gradient variants. How do we provide insights relevant to this whole family of algorithms? Were $J(\pi)$ concave, we would expect that sensible optimization method converge to the solution of $\max_{\pi \in \Pi_{\phi}} J(\pi)$, allowing us to abstract away the details of the optimization procedure and study instead the quality of decisions that can be made using a certain constrained policy class. Unfortunately, $J(\pi)$ is non-concave \citep{bhandari2019global,agarwal2019optimality}. It is, however, smooth (see Lemma~\ref{lem: smoothness}). In smooth non-concave optimization, we expect sensible first-order methods to converge to a first-order stationary point \citep{bertsekas1997nonlinear}. Studying the quality of policies that are stationary points of $J(\cdot)$ then gives broad insight into how the use of  restricted policy classes affects the limiting performance reached by policy gradient methods.

As defined below, a policy is a first-order stationary point if, based on a first-order approximation to $J(\cdot)$, there is no feasible direction that improves the objective value. Local search algorithms generally continue to increase the objective value until reaching a stationary point. A first order stationary point could be a local maximum or a saddle point. The latter presents no additional complications for applying Theorem \ref{thm: main result} below.  Throughout this section, we view each $\pi \in \Pi$ as a stacked vector $\pi = (\pi(s,a) : s \in \Sc, a \in \Ac ) \in \mathbb{R}^{|\Sc| \cdot |\Ac| }$. It may also be natural to view $\pi$ as an $|\Sc|\times |\Ac|$ dimensional matrix whose rows are probability vectors. In that case, all results are equivalent if one views inner products as the standard inner product on square matrices, given by $\langle A, B \rangle={\rm Trace}(A^{\top} B)$, and all norms as the Frobenius norm. 
\begin{defn}
	A policy $\pi \in \Pi$ is a first order stationary point of $J: \Pi \to \mathbb{R}$ on the subset $\Pi_{\phi}\subset \Pi$ if
	\[
	\langle \nabla J(\pi) \, , \, \pi' - \pi \rangle \leq 0  \qquad \forall \pi' \in \Pi_{\phi}.  
	\]
\end{defn}
The following smoothness result is shown by a short calculation\footnote{In Arxiv version 2, this is Lemma E.3. To translate their result to our formulation, one must multiply the statement in Lemma E.3 by $(1-\gamma)$, as in the definition $J(\pi)=(1-\gamma) \rho V_{\pi}$. They also have normalized so that $|r(s,a)|\leq 1$. That is the reason $\|r\|_{\infty}$ does not appear in their expression.} in \cite{agarwal2019optimality}. 

\begin{lemma}
	\label{lem: smoothness} 
	For every $\pi, \pi' \in \Pi$, $\| \nabla J(\pi) - \nabla J(\pi')\|_{2} \leq  L \|  \pi - \pi' \|_{2}$ where $L=\frac{2 \gamma |\Ac| \|r\|_{\infty} }{(1-\gamma)^2 }.$ 
\end{lemma}

Above, we've claimed that we expect first order methods applied to smooth nonconvex optimization are expected to converge to first-order stationary points. This is a standard subject in nonlinear optimization \citep[see e.g.][]{bertsekas1997nonlinear} and the recent  literature has proposed stochastic first order methods with fast convergence rates to stationary points \cite[see e.g.][]{ghadimi2013stochastic,xiao2014proximal, defazio2014saga, reddi2016stochastic, ghadimi2016accelerated,reddi2016proximal, reddi2016stochasticfrank,  davis2019proximally}. Below, as an illustration, we will show a convergence result for an \emph{idealized} policy gradient method, with exact gradient evaluations and a direct parameterization. Armed with such a result, we focus on the quality of stationary points.  

 Recall that a point $\pi_{\infty}$ is a limit point of a sequence if some subsequence converges to $\pi_{\infty}$. Bounded sequences have convergent subsequences, so limit points exist for the sequence $\{ \pi_t \}$ in Lemma \ref{lem: convergence to stationary}. The operator ${\rm Proj}_{2, \Pi_{\phi} }(\pi) = \argmin_{\pi' \in \Pi_{\phi}} \|\pi' - \pi\|_{2}^2$ denotes orthogonal projection onto the convex set $\Pi_{\phi}$. This exact lemma statement can be found in \cite{bhandari2019global} and similar statements appear in nonlinear optimization textbooks  \citep{beck2017first, bertsekas1997nonlinear}. 
\begin{lemma}[Convergence to stationary points]\label{lem: convergence to stationary}
	For any $\pi_1 \in \Pi$ and $\alpha \in \left(0, \frac{1}{L}\right]$, let
	\begin{equation}\label{eq:pgd} 
	\pi_{t+1} = {\rm Proj}_{2, \Pi_{\phi} }\left( \pi_{t} + \alpha \nabla J(\pi_t)\right) \qquad t=1,2,3\cdots
	\end{equation}
	If $\pi_{\infty}$ is a limit point of $\{ \pi_{t}: t\in \N \}$,  then $\pi_{\infty}$ is a stationary point of $J(\cdot)$ on $\Pi_{\phi}$ and $$\lim_{t\to \infty} J(\pi_t) = J(\pi_{\infty}).$$  
\end{lemma}
\begin{remark}[Steepest feasible ascent]\label{rem:prox_view_of_projection} It is well known \citep[see e.g.][]{beck2017first} that policies generated by \eqref{eq:pgd} above satisfy $	\pi^{t+1}  = \argmin_{\pi \in \Pi_{\phi}} \left(\pi^{t} + \langle \nabla J(\pi^{t})\, , \,  \pi-\pi^t  \rangle  + \frac{1}{2\alpha} \left\| \pi-\pi^t \right\|^2_2  \right)$. When the stepsize $\alpha$ is small, a projected gradient update essentially moves in the steepest feasible ascent direction until reaching a stationary point, from which there are no feasible ascent ascent.  	
\end{remark}
\begin{remark}[Practical implementation]
	 The appendix provides many extra details related to this algorithm. It explains that this projection can be computed using simple soft-thresholding operations and that the whole algorithm can be implemented efficiently while storing only a parameter $\theta \in \mathbb{R}^{m\cdot |\Ac|}$. This stores one value per state segment and action, rather than one per state. The appendix also shows how to generate unbiased stochastic gradients of $J(\cdot)$. The body of this paper will instead focus on the quality of the stationary points of $J(\cdot)$, abstracting away the specifics of which policy gradient method is used. 
\end{remark}

\paragraph{Quality of stationary points.}
We will measure the accuracy of a state-aggregation $\phi(\cdot)$ through the maximal difference between state-action values with states belonging to the same segment of the state space. This notion is weaker than alternatives that explicitly assume transition probabilities and rewards are uniformly close within segments --- usually by imposing a smoothness condition \citep[see e.g.][]{rust1996numerical}. But it is a stronger requirement than a the recent one in \cite{dong2019provably}, which only looks at the gap between state-action values under the optimal value function. The current proof requires bounding the aggregation error of some policy $\pi_{\infty}$ that is a stationary point of $J(\cdot)$, so it does not seem possible to give guarantees for policy gradient methods if we replace $Q_{\pi}$ with $Q^*$ in the definition below. At the same time, performance bounds that depend on $\epsilon_{\phi}$ can be quite conservative. See for instance Figures \ref{fig:simulation} and \ref{fig:simulation2}. It is an open question whether this definition can be relaxed in a meaningful way.  \cite{li2006towards} provides a comparison of different measures of the approximation error of a state-aggregated representation. 
\begin{defn}[Inherent state aggregation error]\label{def:error}
	Let $\epsilon_{\phi} \in \mathbb{R}$ be the smallest scalar satisfying 
	\[
	\, | Q_{\pi}(s, a) - Q_{\pi}(s', a) |\leq \epsilon_{\phi}
	\]
	for every $\pi \in \Pi_{\phi}$, $a\in \Ac$, and all $s, s'\in \Sc$ such that $\phi(s)=\phi(s')$. 
\end{defn}

Despite the non-concavity of $J(\cdot)$, one can give guarantees on the quality of its stationary points.  The next result does so under the requirement that each state-space segment has positive probability under the initial weighting $\rho$. Similar assumptions appear in \cite{bhandari2019global} and \cite{agarwal2019optimality} and they each discuss its necessity at some length. While the next result holds for $\rho$ that is nearly degenerate, it should be emphasized that the convergence rates of many policy gradient methods depend inversely on $\min_i \rho\left(\phi^{-1}(i)\right)$. An exploratory initial distribution is critical to these algorithm's practical success. Recall that $J(\cdot)$, as defined in \eqref{eq:objective_J}, is normalized so that it represents average rather than cumulative reward earned.   Recall also that $\pi^* \in \Pi$ denotes some optimal policy, which by definition satisfies $V_{\pi^*}(s)=V^*(s) \,\, \forall s\in \Sc$. Such a $\pi^*$ is also an unconstrained maximizer of the policy gradient objective $J(\cdot)$.
\begin{thm}[Quality of stationary points]\label{thm: main result} Suppose $\rho\left(\phi^{-1}(i)\right)>0$ for each $i \in \{1,\cdots, m\}$. If $\pi_{\infty}$ is a stationary point of $J(\cdot)$ on $\Pi_{\phi}$, then
	\[ 
	J(\pi^*) - J(\pi_{\infty}) \leq (1-\gamma) \|  V_{\pi_{\infty}} - V^* \|_{\infty}  \leq 2\epsilon_{\phi}.
	\]
\end{thm}

 For purposes of comparison, let us provide an alternative result, which can be derived by specializing a result in \cite{bhandari2019global}. At the time when this paper was initially written and posted, the result below was the best available bound. See the remark below for a detailed comparison with contemporaneous statement by \cite{agarwal2019optimality}. See the conclusion for discussion around how the special structure of state-aggregation seems to drive the improvements in Theorem \ref{thm: main result}.
\begin{thm}[Earlier result by \cite{bhandari2019global}]\label{thm:old_result}
	If $\pi_{\infty}$ is a stationary point of $J(\cdot)$ on $\Pi_{\phi}$, then
	\[ 
	J(\pi^*) - J(\pi_{\infty})  \leq  \kappa_{\rho} \frac{\epsilon_{\phi}}{(1-\gamma)}
	\]
	where 
	\[
	\kappa_{\rho} \leq  \max_{i \in \{ 1,\cdots,m \}}  \frac{  \eta_{\pi^*}\left( \phi^{-1}(i) \right) }{ \rho\left( \phi^{-1}(i) \right) }.
	\]
\end{thm}
Here, $\kappa_{\rho}$ captures whether the weight the initial distribution $\rho$ places on each segment of the state partition is aligned with the occupancy measure under an optimal policy. The form here is somewhat stronger than the simple one in \cite{kakade2002approximately}, which does not aggregate across segments, but it is still problematic. Without special knowledge about the optimal policy, it is impossible to guarantee $\kappa_{\rho}$ is smaller than the number of segments $m$.  Worse perhaps is the dependence on the effective horizon $1/(1-\gamma)$. Recall from Section \ref{sec: formulation} that $J(\pi) \in [0,1]$ has the interpretation of a geometric average reward per decision. The optimality gap $	J(\pi^*) - J(\pi_{\infty})$ then represents a kind of average per-decision regret produced by a limiting policy. The dependence on $1/(1-\gamma)$ on the right hand side is then highly problematic, suggesting performance degrades entirely in a long horizon regime. While undesirable, this horizon dependence is unavoidable under some classic approximate dynamic programming procedures. This was shown for approximate value iteration (with a fixed state-weighting) by \cite{van2006performance}. In the next section, we will show this is true for approximate policy iteration as well. 

\begin{remark}[Strengthened results in \cite{agarwal2019optimality}]\label{rem:updates_to_agrawal}
	Initially, a similar result to Theorem \ref{thm:old_result} could be found in \cite{agarwal2019optimality}, although with an even worse dependence on the discount factor.  However, contemporaneously with this paper, their paper's results were updated and stated in terms of a notion called transfer error. Specializing this bound to our setting would show that the limiting optimality gap is bounded by $2\epsilon_{\phi}\sqrt{k}$ under a state-aggregated natural policy gradient method. That bound avoids a poor dependence on the problem's time horizon. However, relative to Theorem \ref{thm: main result} it still has a poor dependence on the number of actions. Such an upper bound cannot be compared cleanly against the lower bound for API we provide in Theorem \ref{thm:api_lower}, so for the purposes of this paper the tighter result in Theorem \ref{thm: main result} is critical. 
\end{remark}

\paragraph{Proof of Theorem \ref{thm: main result}.}
The next lemma is a version of the policy gradient theorem \citep{sutton2018reinforcement} that applies with directly parameterized policies. It is easy to deduce this formula from ones in \cite{agarwal2019optimality}, for example.  The inner product interpretation in the statement is inspired by \cite{konda2000actor}.  For any given state-relevance weights $w \in \mathbb{R}^{|\Sc|}$, define the inner product $\langle \cdot , \cdot \rangle_{w\times 1}$ on $\mathbb{R}^{|\Sc|\times |\Ac|}$ by  
\begin{equation}\label{eq:weighted_inner_product}
	 \langle Q\, , \, Q' \rangle_{w \times 1} = \sum_{s\in \Sc} \sum_{a\in \Ac} w(s)  Q(s,a)Q'(s,a).  
\end{equation}

\begin{lemma}[Policy gradient theorem for directional derivatives]\label{lem: pg thm} For each $\pi, \pi' \in \Pi$, 
	\begin{equation}\label{eq: directional derivatives}
	\langle \nabla J(\pi), \pi' -\pi\rangle = \sum_{s \in \Sc} \sum_{a\in \Ac} \eta_{\pi}(s)   Q_{\pi}(s,a) \left( \pi'(s,a) - \pi(s,a) \right)  = \langle Q_{\pi}\, , \, \pi'-\pi \rangle_{\eta_{\pi} \times 1}
	\end{equation}	
\end{lemma}
\begin{proof}[Proof sketch]
The proof sketch here gives insight into the second-order remainder error term in a first-order Taylor expansion of $J(\cdot)$. We have,
	\begin{align}\nonumber
		J(\pi') - J(\pi) =& \sum_{s \in \Sc} \sum_{a\in \Ac} \eta_{\pi'}(s)   \left( \pi'(s,a) - \pi(s,a)\right) Q_{\pi}(s, a) \\
		=& \langle Q_{\pi}\, , \, \pi'-\pi \rangle_{\eta_{\pi} \times 1}
		+ \underbrace{\sum_{s \in \Sc} \sum_{a\in \Ac} \left(\eta_{\pi'}(s) - \eta_{\pi}(s)\right) \left( \pi'(s,a) - \pi(s,a)\right) Q_{\pi}(s, a)}_{=O(\| \pi'-\pi \|^2 )}. \label{eq:pg_thm_dist_shift}
	\end{align}
	The first equality is a simple but powerful result known in the RL literature as the \emph{performance difference lemma} \citep{kakade2002approximately}. That the remainder term is second order uses that $\pi \mapsto P_{\pi}$ is linear and therefore  $\eta_{\pi} = (1-\gamma) \rho (I-\gamma P_{\pi})^{-1}$ is differentiable in $\pi$. 
\end{proof}

Since $Q(s,d)=\sum_{a} Q(s,a) d_a$ for any action distribution $d\in \Delta(\Ac)$, this formula could be written as $\langle \nabla J(\pi), \pi' -\pi\rangle =  \E_{S\sim \eta_{\pi}}\left[ Q_{\pi}(S, \pi'(S) ) - Q_{\pi}(S, \pi(S) ) \right]$. One can interpret $Q_{\pi}(s,a)$ as measuring the benefit of switching from $\pi$ to action $a$ for a single period. The policy gradient formula \eqref{eq: directional derivatives} says that the infinite-horizon impact of a local policy change in the direction of $\pi'$ is equal to the average benefit of switching to policy $\pi'$ for a single period and at a random state. 

The next lemma is a special case of one in \cite{bhandari2019global}. This simplified setting allows for an extremely simple proof, so we include it for completeness. Equation \eqref{eq:bellman_for_stationary} can be viewed as an approximate Bellman equation within the restricted class of policies. 
\begin{lemma}[An approximate Bellman equation for stationary points]
	\label{lem: Bellman for stationary}
	If $\pi_{\infty}$ is a stationary point of $J(\cdot)$ on $\Pi_{\phi}$, then 
	\begin{equation}\label{eq:bellman_for_stationary}
	\E[V_{\pi_{\infty}}(S)] = \max_{\pi \in \Pi_{\phi}} \E\left[ T_{\pi} V_{\pi_{\infty}}(S)\right]  \quad \text{where } S\sim \eta_{\pi_{\infty}}. 
	\end{equation}
\end{lemma}
\begin{proof}
	Continue to let $S$ denote a random draw from $\eta_{\pi_{\infty}}$. For every $\pi \in \Pi_{\phi}$ we have,  
	\begin{align*}
	0  \geq \langle \nabla J(\pi_{\infty}) \, , \,  \pi - \pi_{\infty} \rangle = \langle Q_{\pi_{\infty}}\, , \, \pi- \pi_{\infty} \rangle_{\eta_{\pi_{\infty}} \times 1} &=  \E\left[  Q_{\pi_{\infty}}(S,  \pi(S))) - Q_{\pi_{\infty}}(S,   \pi_{\infty}(S) ) \right] \\
	&= \E\left[ \left(T_{\pi} V_{\pi_{\infty}}\right)(S) -  V_{\pi_{\infty}}(S)  \right].
	\end{align*}
	The second equality uses \eqref{eq: Q to bellman}. The reverse inequality uses that $\pi_{\infty} \in \Pi_{\phi}$ along with the Bellman equation $V_{\pi_\infty} = T_{\pi_{\infty}}V_{\pi_{\infty}}$. 
\end{proof}

We are now ready to prove Theorem \ref{thm: main result}. 
\begin{proof}[Proof of Theorem \ref{thm: main result}] We apply Lemma \ref{lem: Bellman for stationary} and several times use the connection between  $Q$ functions and Bellman operators in \eqref{eq: Q to bellman}. For notational convenience, throughout let $S$ denote a random draw from $\eta_{\pi_{\infty}}$ and let $\Sc_i=\phi^{-1}(i)$ denote the $i^{\text{th} }$ segment of the state space. Since $\E\left[T_{\pi_{\infty}} V_{\pi_{\infty}}(S)\right] = \max_{\pi \in \Pi_{\phi}} \E\left[ T_{\pi} V_{\pi_{\infty}}(S)\right]$, we have 
	\begin{align*}
	\pi_{\infty} \in \argmax_{\pi \in \Pi_{\phi}} \E\left[T_{\pi} V_{\pi_{\infty}}(S)\right] &= \argmax_{\pi \in \Pi_{\phi}} \E\left[ Q_{\pi_{\infty}}(S, \pi(S) ) \right]\\&= \argmax_{\pi \in \Pi_{\phi}} \sum_{i=1}^{m} \E\left[Q_{\pi_{\infty}}(S,\pi(S)) \mid S\in \Sc_i \right] \Prob(S \in \Sc_i).
	\end{align*}
	Let $a^{\infty}_{i}$ denote the action selected by policy $\pi_{\infty}$ at any state $s \in \phi^{-1}(i)$ in segment $i$. The vector $(a^{\infty}_1, \cdots, a^{\infty}_{m})$ provides a full description of the policy $\pi_{\infty}$. The optimization problem above decomposes across sates, showing 
	\[ 
	a^{\infty}_i \in \argmax_{a\in \Ac} \E\left[ Q_{\pi_{\infty}}(S,a) \mid S \in \Sc_i\right] \qquad i=1,\cdots, m. 
	\]
	Here we use implicitly that $\Prob(S \in \Sc_i)>0$, which is assured by our assumption that $\rho(\Sc_i)>0$. 
	Now, we use the definition of $\epsilon_{\phi}$ to show $a^{\infty}_i$ must be near optimal at every state in partition $i$. 
	Pick 
	\[
	(s_i^*, a_i^*) \in \argmax_{s\in \Sc_i, a \in \Ac} Q_{\pi_{\infty}}(s,a).
	\]
	By the optimality of $a_{i}^{\infty}$ there must exist some $\tilde{s} \in \Sc_i$ such that $Q_{\pi_{\infty}}(\tilde{s}, a_i^{\infty}) \geq  Q_{\pi_{\infty}}(\tilde{s}, a_i^*)$. For any other $s\in \Sc_i$ we have
	\[ 
	Q_{\pi_{\infty}}(s, a_i^\infty) \geq Q_{\pi_{\infty}}(\tilde{s}, a_i^{\infty}) - \epsilon_{\phi} \geq Q_{\pi_{\infty}}(\tilde{s}, a^*_i) - \epsilon_\phi \geq  Q_{\pi_{\infty}}(s^*_i, a^*_i) - 2\epsilon_{\phi} \geq \max_{a \in \Ac} Q_{\pi_{\infty}}(s, a) - 2\epsilon_{\phi}. 
	\]
	Observe that  $Q_{\pi_{\infty}}(s, a_i^\infty)  = Q_{\pi_{\infty}}(s, \pi_{\infty}(s)) = V_{\pi_\infty}(s)$ and $\max_{a \in \Ac} Q_{\pi_{\infty}}(s, a)= TV_{\pi_{\infty}}(s)$. Since $s$ is arbitrary, this gives element-wise inequality $V_{\pi_{\infty}} \succeq TV_{\pi_{\infty}} - 2\epsilon_{\phi} e$ where $e$ denotes a vector of ones. Using the monotonicity of Bellman operators and the fact that $T(V+ce)=TV+\gamma ce$ \citep{bertsekas1995dynamic}, we have 
	\[
	V_{\pi_{\infty}} \succeq TV_{\pi_{\infty}} - 2\epsilon_{\phi} e \succeq  T^2 V_{\pi_{\infty}} + 2\gamma \epsilon_{\phi} e-2\epsilon_{\phi} e\succeq \cdots \succeq  V^* - \frac{2\epsilon_{\phi}}{1-\gamma}  e .
	\]
\end{proof}

\section{Horizon dependent per-period regret under API}
\label{sec: api}
Approximate policy iteration is one of the classic approximate dynamic programming algorithms. It has deep connections to popular methods today, like $Q$-learning with target networks that are infrequently updated \citep{mnih2015human}. Approximate policy iteration is presented in Algorithm \ref{alg: api}. The norm there is the one induced by the inner product in \eqref{eq:weighted_inner_product}, defined by $\| Q\|_{2, w\times 1}= \sqrt{ \sum_{s} \sum_{a} w(s) Q(s,a)^2}$. The procedure mimics the  classic policy iteration algorithm \citep{puterman2014} except it uses a regression based approximation in the policy evaluation step, aiming to select a state-aggregated value function that is close to the true one in terms of mean squared error. It is worth noting that this is a somewhat idealized form of the algorithm. Practical algorithms use efficient sample based approximations to the least-squares problem defining $\hat{Q}$. See \cite{bertsekas1996neuro} or \cite{bertsekas2011approximate} for an introduction.

How does this algorithm perform? Our main result in this section is captured by the following proposition, giving a lower bound on performance which is worse than the result in Theorem \ref{thm: main result} by a factor of the effective horizon $1/(1-\gamma)$. Recall from Section \ref{sec: state agg} that there is a broad equivalence between searching over the restricted class of value functions in $\Qc_{\phi}$ and searching over the restricted class of policies $\Pi_{\phi}$. Any advantage in the limiting performance of policy gradient methods is due to the way in which it searches over policies and not an advantage in representational power. 
\begin{thm}\label{thm:api_lower}
	There exists an MDP, a state-aggregation $\phi$, and initial policy $\pi_1$ such that if 
	$\left\{\pi_t\right\}_{t\in \mathbb{N}}$ is generated by Algorithm \ref{alg: api} with inputs given by $\phi$, $\pi_1$, and uniform weighting $w(s)=1/|\Sc| \,\, \forall s$, then
	\begin{equation}\label{eq: api lower} 
	\liminf_{t\to \infty} \,\, J(\pi^*) - J(\pi_{t}) \geq \frac{ \gamma \epsilon_{\phi}/4 }{(1-\gamma)}.
	\end{equation}
\end{thm}
Later in this section, we study a variant of API which adapts state-relevance weights across iterations.

Note that a classic result of  \cite[Prop.~ 6.2]{bertsekas1996neuro}, specialized to this setting, can be used to show a bound in the other direction that matches \eqref{eq: api lower} up to a numerical constant. Technically, the analysis of \cite{bertsekas1996neuro} applies to value functions and not state-action value functions. The reader can find the same proof written in terms of state-action value functions in \cite{agarwal2019reinforcement}.

 Theorem \ref{thm:api_lower} is established through Example \ref{counterexample} below, which synthesizes an example from \cite{bertsekas1996neuro} with an  example of \cite{van2006performance}. The latter work studies approximate value iteration in state-aggregated problems, establishing a result like Theorem \ref{thm:api_lower}. The former work gives an example in which the optimality gap under API exhibits poor dependence on the horizon, but that example does not treat state-aggregation. Surprisingly, although the state-aggregated problem in Example \ref{counterexample} is similar to the example of \cite{bertsekas1996neuro}, the conclusion is stronger. The result of  \cite{bertsekas1996neuro} is analogous to \eqref{eq: api lower}, but with the limit-inferior replaced with a limit-superior. API will cycle endlessly between policies, sometimes selecting the optimal policy and sometimes selecting disastrous ones. In addition to applying in state-aggregated problems, Theorem \ref{thm:api_lower} strengthens the conclusion by showing that API consistently selects disastrous policies in the limit.

\begin{minipage}{0.46\textwidth}
	\centering
	\includegraphics[width=.95\linewidth]{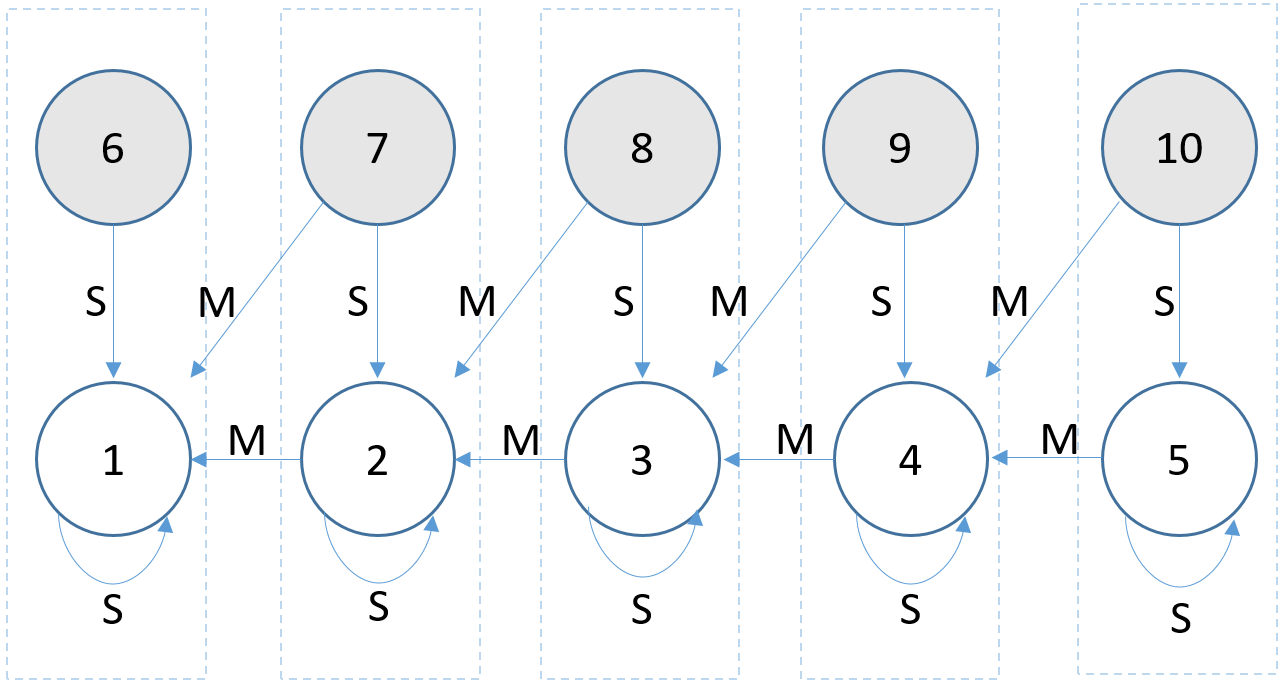}
	\captionof{figure}{A bad example for API. The actions $\Move$ and $\Stay$ are denoted by $M$ and $S$.}
	\label{fig: counterexample} 
\end{minipage}
\hfill
\begin{minipage}{0.46\textwidth}
	\hrulefill

	\begin{algorithm}[H]
		\caption{ API}
		\label{alg: api}
		\SetNlSty{texttt}{(}{)}
		\SetAlgoLined
		\SetKwInOut{Input}{input}\SetKwInOut{Output}{output}
		\Input{$w\in \Delta(\Sc)$, $\pi_1\in \Pi$, $\phi$}
		\BlankLine
		\For{$t=1,2, \cdots, $}{ 
			\tcc{Approximate policy evaluation step}
			$\hat{Q}_t \in \argmin_{\hat{Q} \in \Qc_{\phi}} \| \hat{Q}  - Q_{\pi_t} \|_{2,w\times 1}$ \;
			\tcc{Policy improvement step}
			$\pi_{t+1}(s) \in \argmax_{a\in \Ac} \hat{Q}_{t}(s,a) \, \forall s$\; 
			
		}
	\end{algorithm}\DecMargin{1em}
	
\end{minipage}

\begin{example}\label{counterexample}
Consider an MDP with $n=2m$ states, depicted in Figure \ref{fig: counterexample} for $n=10$ and $m=5$. For $s\in \{ 1,\cdots, m \}$, we have $\phi(s)= \phi(s+m)=s$. This means that the algorithms don't distinguish between $s$ and $s+m$. In state $s\in \{2,\cdots, m\}$ there are two possible actions, $\Move$, which moves the agent to state $s-1$ and generates a reward $r(s,\Move)=0$, and $\Stay$,  which keeps the agent in the same state with reward $r(s, \Stay)$.  Rewards obey the recursion
\[
r(1,\Stay)  =0 \qquad  r(s, \Stay)=\gamma r(s-1, \Stay) - c \qquad \text{for } s\in \{2, \cdots, n \},
\]
and the formula $r(s, \Stay) = -c\sum_{i=2}^{s} \gamma^{i-2}$. The negative reward for the action $\Stay$  can be thought of as a cost. State $1$ has only the costless action \Stay . (Or one can think of \Move as being identical to  \Stay in state 1).

Transition probabilities from state $s+m$ are identical to those at state $s$, and $r(s+m, \Move)=r(s, \Move)=0$, but $r(s+m, \Stay) =r(s, \Stay)+ \epsilon_{\phi}$ where $\epsilon_{\phi}>0$. Pick $c= \epsilon_{\phi}/2$. The optimal policy plays $\Move$ from every state $s\in \{2, \cdots, m\}$. 
\end{example}

Figure \ref{fig:simulation} displays  simulation results. The dashed blue line represents $2\epsilon_{\phi}$, the upper bound on the limiting optimality gap proved in Theorem \ref{thm: main result}. In particular, any optimization method that is guaranteed to reach a stationary point of $J(\cdot)$ will have a limiting optimality gap below the dashed blue line. The simulation results show that a particular variant, projected gradient ascent with small constant stepsize, in fact converges gracefully to optimality in this example. The performance of approximate policy iteration is far worse. In the second iteration, it actually reaches an optimal policy, but from there performance continues to degrade. In the limit, it cycles endlessly between two policies. That cycling behavior is common with API and is confirmed analytically below.

\begin{figure}
	\includegraphics[width=.6\linewidth, center]{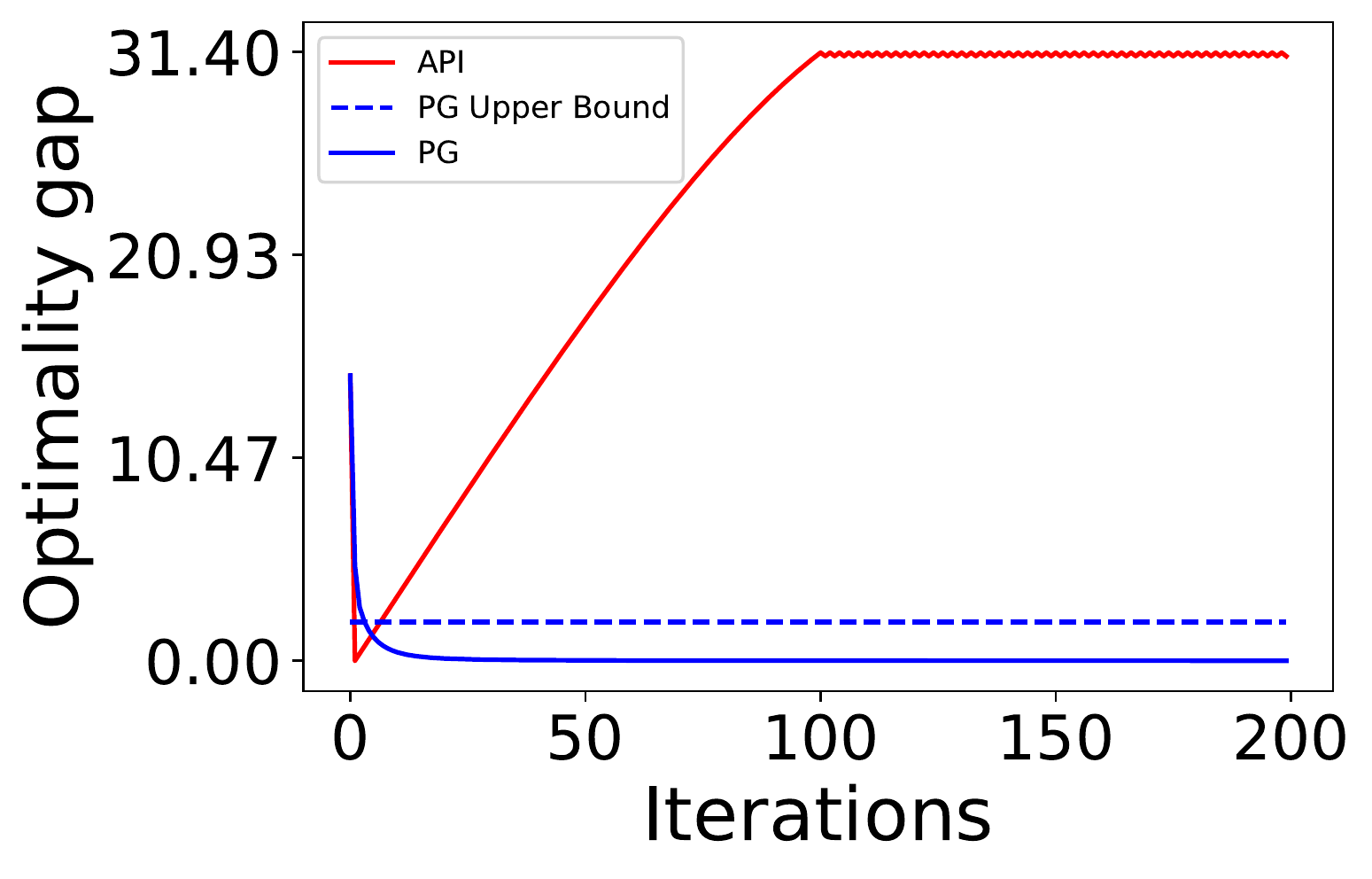}
 \caption{Performance of approximate policy iteration and policy gradient in Example \ref{counterexample}, with $n=200$ sates, discount factor $\gamma=0.99$ and $\epsilon_{\phi}=1$.  }
\label{fig:simulation}
\end{figure}

\begin{proof}[Proof of Theorem \ref{thm:api_lower}]
	Consider API applied to Example \ref{counterexample} with uniform weighting, i.e. $w(s)=1/|\Sc|$, and an initial policy $\pi_1$ with $\pi_1(s)=\Stay$ for $s\in \{2,4, 6, \cdots, m-1 \}$ and $\pi_1(s)=\Move$ for $s\in \{3, 5,7, \cdots, m \}.$ For simplicity, we have assumed $m$ is an odd number. The policy is state-aggregated, so $\pi_1(s+m)=\pi_1(s)$ for $s\leq m$. We show that the next policy produced by API, $\pi_2$, will play $\Move$ at states $\{2,4, \cdots, m-1 \}$ but play $\Stay$ at states $\{3, 5,7, \cdots, m\}$. Proceeding in this manner, one finds that $\pi_3=\pi_1$, $\pi_4=\pi_2$, and the policies cycle endlessly. 
	
	Assuming for the moment that this result holds, let us consider the optimality gap under any initial distribution with $\rho(m)>1/2$. We find $J(\pi^*)-J(\pi_t) > (1/2)(1-\gamma) \left(V^*(m)-V_{\pi_t}(m)\right)$. Observe that an optimal policy $\pi^*$, which chooses $\Move$ in every state, incurs cost $V_{\pi^*}(m)=0$. On the other hand, $V_{\pi_t}(m) = \frac{r(m, \Stay)}{1-\gamma}$ for $t$ even and $V_{\pi_t}(m)=0 +\gamma \cdot \frac{r(m-1, \Stay)}{1-\gamma}$ for $t$ odd. These formulas reflect that either policy moves at most once before staying perpetually at one of the rightmost states. We find
	\[
	J(\pi^*)-J(\pi_t) > -\frac{\gamma}{2} \cdot r(m-1, \Stay) = \frac{c \gamma}{2} \cdot \sum_{i=2}^{m-1} \gamma^{i-2} \overset{m\to \infty}{\longrightarrow} \frac{c \gamma}{2(1-\gamma)} = \frac{\epsilon_{\phi} \gamma}{4(1-\gamma)}.
	\]
	This establishes that Theorem \ref{thm:api_lower} holds for problem instances with $m$ sufficiently large. 
	
	We now turn to verifying that policies cycle in the manner described above. The weighted least-squares problem solved by API has a particularly simple form in this case. It is straightforward to show that the problem decomposes across segments of the state space and, as the conditional mean minimizes squared loss, has the form
	\begin{align*}
		\hat{Q}_{t}(s,a)  =  \E_{S\sim w} \left[ Q_{\pi_t}\left(S, a \right) \mid  S\in \phi^{-1}(s) \right] &= \frac{ Q_{\pi_t}(s,a) + Q_{\pi_t}(s+m, a) }{2}  \quad \forall s\leq m, a\in \mathcal{A}\\
		&=Q_{\pi_t}(s,a) +\left(\nicefrac{\epsilon_\phi}{2}\right)\mathbf{1}(a=\Stay ) \quad \forall s\leq m, a\in \mathcal{A}.
	\end{align*}
	That is, in each segment the value function of the current policy is overestimated by  $\nicefrac{\epsilon_\phi}{2}$ at state $s$ and under-estimated by $\nicefrac{\epsilon_\phi}{2}$ in state $s+m$. 
	
	We verify that $\pi_2$ has the form conjectured above. The proof for $\pi_3$ is uses the same ideas. Under $\pi_1$ and for $s\in \{4, 6, 8, \cdots\}$ we have $V_{\pi_1}(s) = r(s, \Stay) / (1-\gamma)$  and $V_{\pi_1}(s-1)=0+\gamma V_{\pi_1}(s-2) = \gamma r(s-2, \Stay) / (1-\gamma)$. Then,
	\begin{align*}
	Q_{\pi_1}(s, \Stay ) &=  r(s, \Stay ) + \gamma V_{\pi_1}(s)  = r(s,\Stay)/(1-\gamma)\\
	Q_{\pi_1}(s, \Move ) &=  r(s, \Move) + \gamma V_{\pi_1}(s-1) = \gamma^2 r(s-2, \Stay)/(1-\gamma). 
	\end{align*}
   Then, the least-squares approximation gives $\hat{Q}_1(s, \Stay) = r(s,\Stay)/(1-\gamma)  +\epsilon_{\phi}/2$ and $\hat{Q}_{1}(s, \Move)= \gamma^2 r(s-2, \Stay)/(1-\gamma)$. By plugging in $\epsilon_{\phi}= c/2$, one can verify that $\hat{Q}_{1}(s, \Move)> \hat{Q}_{2}(s, \Stay)$, so $\pi_2$ will play $\Move$ from states $s\in \{4, 6, 8, \cdots m \}$. The edge case of $s=2$ needs to be handled separately. One find $\hat{Q}_1(2,\Move)=0$ and $\hat{Q}_{1}(2, \Stay) = -c+ \epsilon_\phi/2 = 0$.  Breaking ties\footnote{This example could be easily modified by taking $c$ to be infitesimally smaller than $\epsilon_{\phi}/2$, in which case no tie breaking mechanism is needed.}  in favor of the action $\Stay$, we get $\pi_2(s) = \Stay$ for $s\in \{2, 4, \cdots, m\}$. 
   
   Under $\pi_1$ and for $s\in \{3, 5, 7,\cdots\}$ we have $V_{\pi_1}(s-1) = r(s-1, \Stay)/(1-\gamma)$ and $ V_{\pi_1}(s) = 0 + \gamma V_{\pi_1}(s-1) =  \gamma r(s-1, \Stay) / (1-\gamma)$. Then,  
   \begin{align*}
   	Q_{\pi_1}(s, \Stay ) &=  r(s, \Stay ) + \gamma V_{\pi_1}(s)  = r(s,\Stay) + \gamma^2 r(s-1, \Stay) / (1-\gamma)\\
   	Q_{\pi_1}(s, \Move ) &=  r(s, \Move) + \gamma V_{\pi_1}(s-1) = \gamma r(s-1, \Stay)/(1-\gamma). 
   \end{align*}
   Then, the least-squares approximation gives $\hat{Q}_1(s, \Stay) =  Q_{\pi_1}(s,\Stay)+\epsilon_{\phi}/2$ and $\hat{Q}_{1}(s, \Move)= Q_{\pi_1}(s, \Move)$. Now, the mis-estimation error $\epsilon_\phi/2$ is enough cause the algorithm to select the decision $\Stay$. In particular, 
   \begin{align*}
   \hat{Q}_1(s, \Stay) - \hat{Q}_1(s, \Move) &= \frac{\epsilon_{\phi}}{2} + r(s, \Stay) - \gamma r(s-1, \Stay)\\
   &=  \frac{\epsilon_{\phi}}{2} + \left(-c+ \gamma r(s-1, \Stay)\right)  - \gamma r(s-1, \Stay)  = \frac{\epsilon_{\phi}}{2} -c = 0.     
   \end{align*} 
   Breaking ties in favor of the action $\Stay$, we find that $\pi_2$ plays $\Stay$ in states $\{3,5,7,\cdots\}$. 
 \end{proof}

\subsection*{Brittle behavior of API with on-policy state relevance weights}

We have illustrated that policy gradient sometimes dramatically outperforms a version of API that uses a fixed state-weighting. Algorithm \ref{alg: api-adaptive} presents another natural form of API in which these weights are adapted over time. At iteration $t$, it weighs states according to the occupancy measure $\eta_{\pi_t}$, prioritizing accuracy at states that are visited often. This choice arises organically if the data used to approximate $Q_{\pi_t}$ is generated by applying $\pi_t$ in the environment.

This modification to API seems to address Example \ref{counterexample}, but it exhibits similarly brittle behavior in other examples. This possibility is validated through the numerical simulation of Example \ref{counterexample2},  depicted in Figure \ref{fig:simulation2}.

\begin{minipage}{0.46\textwidth}
	\centering
	\includegraphics[width=.95\linewidth]{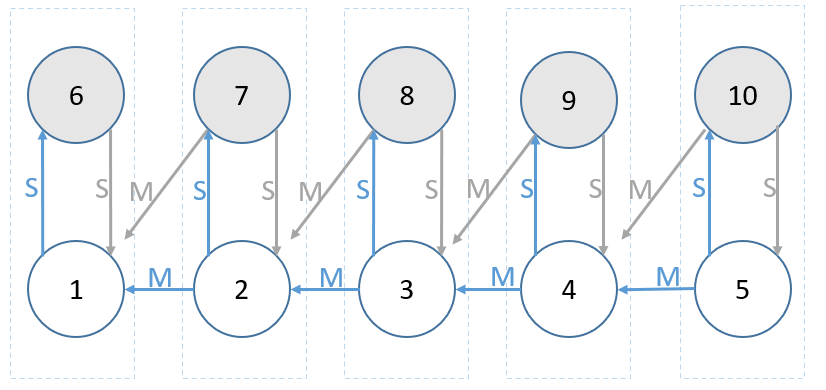}
	\captionof{figure}{A bad example for API with adaptive weights. The actions $\Move$ and $\Stay$ are denoted by $M$ and $S$. }
	\label{fig: counterexample_2} 
\end{minipage}
\hfill
\begin{minipage}{0.48\textwidth}
	\hrulefill

	\begin{algorithm}[H]
		\caption{ API with adaptive state weighting}
		\label{alg: api-adaptive}
		\SetNlSty{texttt}{(}{)}
		\SetAlgoLined
		\SetKwInOut{Input}{input}\SetKwInOut{Output}{output}
		\Input{$w\in \Delta(\Sc)$, $\pi_1\in \Pi$, $\phi$}
		\BlankLine
		\For{$t=1,2, \cdots, $}{ 
			\tcc{policy evaluation}
			$\hat{Q}_t \in \argmin_{\hat{Q} \in \Qc_{\phi}} \| \hat{Q}  - Q_{\pi_t} \|_{2,\eta_{\pi_t}\times 1}$ \;
			\tcc{Policy improvement}
			$\pi_{t+1}(s) \in \argmax_{a\in \Ac} \hat{Q}_{t}(s,a) \, \forall s$\; 
			
		}
	\end{algorithm}\DecMargin{1em}
	
\end{minipage}

\begin{example}\label{counterexample2}
	Consider an MDP with $n=2m$ states, depicted in Figure \ref{fig: counterexample_2} for $n=10$ and $m=5$. For $s\in \{ 1,\cdots, m \}$, we have $\phi(s)= \phi(s+m)=s$. This means that the algorithms don't distinguish between $s$ and $s+m$. In a state with $\phi(s)\in \{2,\cdots, m\}$ there are two possible actions, $\Move$, which moves the agent to state $s-1$ and $\Stay$,  which keeps the agent in the same state partition but brings an agent in state $s$ to $s+m$ an one who is in state $s+m$ to state $s$. The action $\Move$ generates zero reward. 
	Rewards for the action $\Stay$ at the states depicted on the top of Figure \ref{fig: counterexample_2} obey the recursion 
	\[
	r(1+m,\Stay)  =0 \qquad  r(s+m, \Stay)=\gamma r(s+m-1, \Stay) - c \qquad \text{for } s\in \{2, \cdots, m \},
	\]
	and the formula $r(s+m, \Stay) = -c\sum_{i=2}^{s} \gamma^{i-2}$. Playing the action $\Stay$ generates a higher reward in the states depicted at the bottom of Figure \ref{fig: counterexample_2}, with $r(s,\Stay)=r(s+m, \Stay)+\epsilon_{\phi}$.
\end{example}

Example \ref{counterexample2}, is designed so that Algorithm \ref{alg: api-adaptive} behaves just like Algorithm \ref{alg: api} does in Example \ref{counterexample}. Two careful modifications to the example ensure this.  First, in Example \ref{counterexample2} selecting $\Stay$ repeatedly causes the system to cycle between two states in a common partition. Second, the reward generated from playing $\Stay$ is higher in a state $s\in\{2,\cdots, m\}$  (bottom row of Figure \ref{fig: counterexample_2}) than in the corresponding state $s+m$  (top row of Figure \ref{fig: counterexample_2}).

As in the proof of Theorem \ref{thm:api_lower}, imagine Algorithm \ref{alg: api-adaptive} is applied in Example \ref{counterexample2} with initial policy $\pi_1$ that selects $\pi_1(s)=\Stay$ for $s\in \{2,4, 6, \cdots, m-1 \}$ and $\pi_1(s)=\Move$ for $s\in \{3, 5,7, \cdots, m \}$. One can show, and this is validated numerically with open-source code, that the next policy produced by API, $\pi_2$, will play $\Move$ at states $\{2,4, \cdots, m-1 \}$ but play $\Stay$ at states $\{3, 5,7, \cdots, m\}$. This cycle continues, with $\pi_3=\pi_1$, $\pi_4=\pi_2$, and so on. The performance of these policies is depicted in Figure \ref{fig:simulation2}.

\begin{figure}
	\includegraphics[width=.6\linewidth, center]{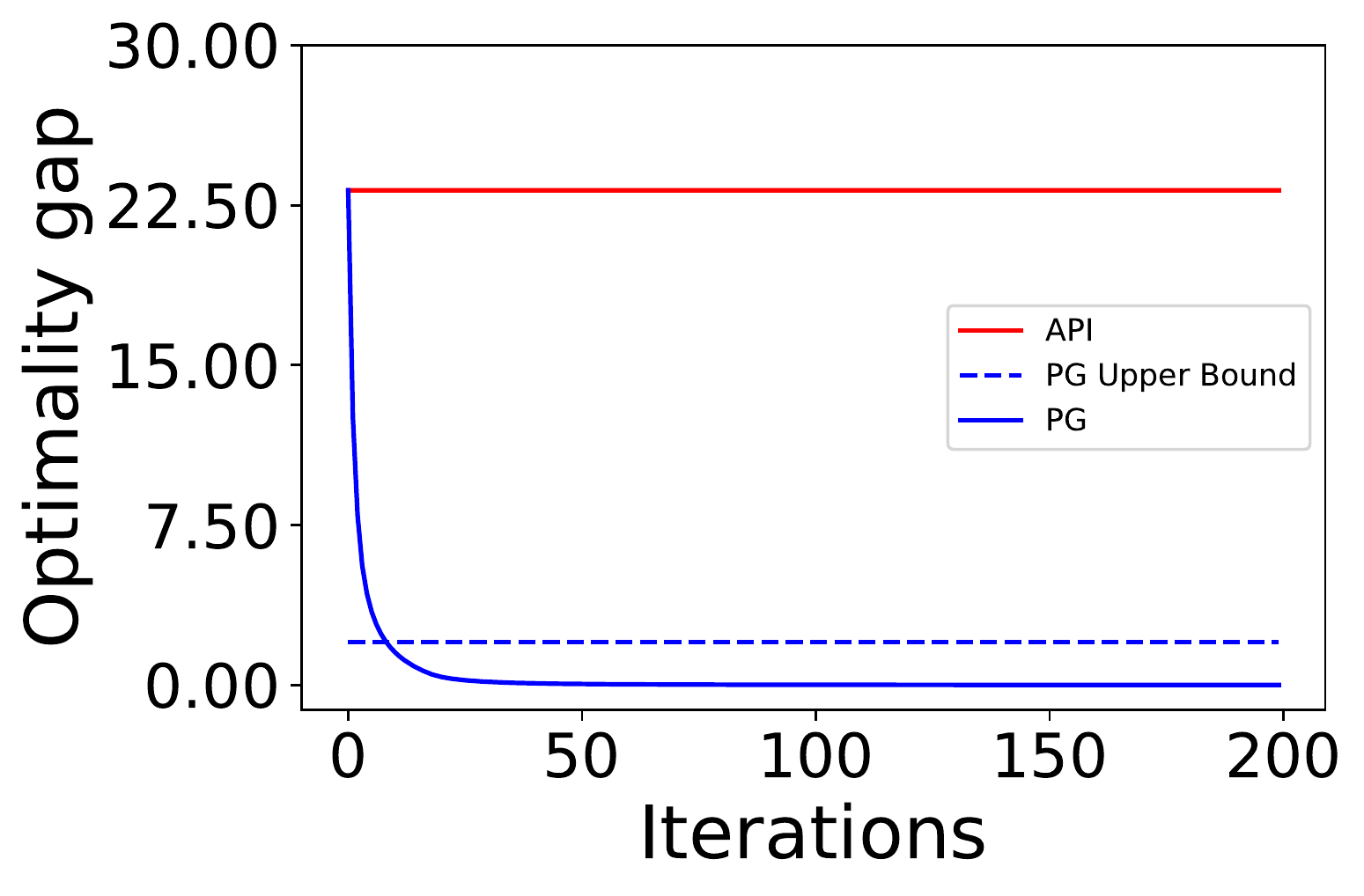}
	\caption{Performance of Algorithm \ref{alg: api-adaptive} and policy gradient in Example \ref{counterexample2}, with $n=400$ sates, discount factor $\gamma=0.99$, $\epsilon_{\phi}=1$, and $c=1/3$.  The initial distribution has the form $\rho(s)=C\cdot 20s$ and $\rho(s+m)=C\cdot s$ where $s\in \{1,\cdots, m\}$ and $C$ is a normalizing constant.  }
	\label{fig:simulation2}
\end{figure}

What drives this? Consider the policy $\pi_1$ described above and some $s\in \{3,5,7,\cdots\}$, so $\pi_1(s)=\Move$. Policy iteration considers the value of deviating from the prescribed action of $\pi_1$ for \emph{only a single period}. So the agent is essentially comparing (A) picking  $\Move$  and then continually selecting $\Stay$ at states $\{s-1, s-1+m\}$ and (B) picking $\Stay$ and transitioning to state $s+m$, then moving to $s-1$ and selecting $\Stay$ thereafter. Because the initial distribution is much more likely to place the agent at $s$ (top of Fig.~\ref{fig: counterexample_2}) than at $s+m$ (bottom of Fig.~\ref{fig: counterexample_2}) and agent plays $\Move$, API with on-policy state-weighting essentially ignores behavior at $s+m$ when fitting an approximate $Q$-function. The problem is constructed so that there is a higher reward to playing $\Stay$ at state $s$ than at $s+m$, and this is enough to cause the agent to estimate that (B) is preferable to (A). The reverse of this logic plays out for states $s\in \{2,4,6,\cdots\}$ with $\pi_1(s)=\Stay$. This argument is validated numerically. An analytical proof is likely similar to the argument establishing Theorem \ref{thm:api_lower} and is omitted.

\section{What drives the performance gap? The importance of on-policy state-weighting and incremental updates.} \label{sec:actor_critic}

Policy gradient methods are intimately related to API, making the performance difference between them all the more striking. To make their connections clear, we establish in Theorem \ref{thm:pg-is-api} a precise equivalence between two algorithms: one is a Frank-Wolfe \citep{frank1956algorithm, jaggi2013revisiting} variant of policy gradient and the other is a form of API which uses online state-relevance weights and soft policy updates. After giving a short proof, we turn to discussion the insights this equivalence yields. 
%
%
%

\begin{thm}\label{thm:pg-is-api}
	Suppose Algorithms \ref{alg: fw-pg} and \ref{alg: soft-api} are applied with the same inputs and the optimization problems in step (2) of Algorithm \ref{alg: fw-pg} and step (3) of Algorithm \ref{alg: soft-api} always have unique solutions. Then each algorithm produces an identical sequence of policies.   
\end{thm}

\begin{minipage}{0.52\textwidth}
	\begin{algorithm}[H]
		\caption{Frank-Wolfe Policy Gradient}
		\label{alg: fw-pg}
		\SetNlSty{texttt}{(}{)}
		\SetAlgoLined
		\SetKwInOut{Input}{input}\SetKwInOut{Output}{output}
		\Input{$\alpha, \pi_1\in \Pi$, $\phi$}
		\BlankLine
		\For{$t=1,2, \cdots, $}{ 
			\tcc{Maximize linearization}
			$\tilde{\pi}_{t+1} = \argmax_{\pi \in \Pi_{\phi}} \langle \nabla J(\pi_t) \, , \,  \pi-\pi_t \rangle $ \;
			\tcc{Soft policy update}
			$\pi_{t+1} = \alpha \tilde{\pi}_{t+1} + (1-\alpha) \pi_t$\;  
		}
	\end{algorithm}\DecMargin{1em}
	
\end{minipage}
\hfill
\begin{minipage}{0.48\textwidth}
	\hrulefill

	\begin{algorithm}[H]
		\caption{Soft API with adaptive weighting}
		\label{alg: soft-api}
		\SetNlSty{texttt}{(}{)}
		\SetAlgoLined
		\SetKwInOut{Input}{input}\SetKwInOut{Output}{output}
		\Input{$\alpha,\pi_1\in \Pi$, $\phi$}
		\BlankLine
		\For{$t=1,2, \cdots, $}{ 
			\tcc{Occupancy weighted policy evaluation }
			$\hat{Q}_t \in \argmin_{\hat{Q} \in \Qc_{\phi}} \| \hat{Q}  - Q_{\pi_t} \|_{2,\eta_{\pi_t}\times 1}$ \;
			\tcc{Policy improvement}
			$\tilde{\pi}_{t+1}(s) \in \argmax_{a\in \Ac} \hat{Q}_{t}(s,a) \, \forall s$\; 
			\tcc{Soft policy update}
			$\pi_{t+1} = \alpha \tilde{\pi}_{t+1} + (1-\alpha) \pi_t$\;  
		}
	\end{algorithm}\DecMargin{1em}	
\end{minipage}

Theorem \ref{thm:pg-is-api} appears to be new, but related results have appeared several times in the literature. Algorithm \ref{alg: api-adaptive} is often called \emph{conservative policy iteration} and was first proposed by \cite{kakade2002approximately} based on considerations similar to policy gradient methods. In cases without approximation, \cite{vieillard2019connections} and \cite{bhandari2021linear} observed that the Frank-Wolfe algorithm is equivalent to Algorithm \ref{alg: soft-api}. \cite{o2017combining} and \cite{schulman2017equivalence} study a related equivalence when entropy regularization is applied.
	

\subsection{Proof of Theorem \ref{thm:pg-is-api} using   the actor-critic theorem}
To compare policy gradient and API, we rely on the theory of actor-critic methods, which use estimated value functions in evaluating gradients of $J(\cdot)$. To make this precise, recall the policy gradient expression in Lemma \ref{lem: pg thm} expresses directional derivatives as a certain weighted inner product, 
$\langle \nabla J(\pi), \pi' -\pi\rangle = \langle Q_{\pi}\, , \, \pi'-\pi \rangle_{\eta_{\pi} \times 1}$.
Actor critic methods replace the true value function $Q_{\pi}$ with some parametrized approximation, producing an approximate gradient. 

An extremely elegant result of \cite{konda2000actor} and \cite{sutton2000policy} shows that compatible value function approximation produces \emph{no error} in evaluating the gradient in feasible ascent directions. Below we identify the form of compatible function approximation in our setting. As in the previous section, $ \| Q  \|_{2, {\eta_{\pi} \times 1}}$ denotes the norm induced by the inner product $ \langle\cdot\, ,\, \cdot \rangle_{\eta_{\pi} \times 1}$ defined in \eqref{eq:weighted_inner_product}.  
\begin{lemma}[Compatible function approximation]\label{lem: compatible approx}
	If $\hat{Q}_{\pi} = \argmin_{\hat{Q} \in \Qc_{\phi}} \| \hat{Q} - Q_{\pi}  \|_{2, {\eta_{\pi} \times 1}} $, then, 
	\[
	\langle \nabla J(\pi),\, \pi' -\pi\rangle =  \langle \hat{Q}_{\pi}\, , \, \pi'-\pi \rangle_{\eta_{\pi} \times 1} \,\, \forall \pi' \in \Pi_{\phi}. 
	\]
\end{lemma}
\begin{proof}
	Observe that $\Qc_{\phi} = {\rm Span}\left( \Pi_{\phi} \right)$, where ${\rm Span}\left( \Pi_{\phi} \right)$ consists of all vectors of the form $\sum_{i=1}^{I} c_i \pi^{(i)} $ where each $c_i \in \mathbb{R}$ is a scalar and $\pi^{(i)} \in \Pi_{\phi}$. Then, $\hat{Q}_{\pi}$ is the orthogonal projection of $Q_{\pi}$ onto ${\rm Span}\left( \Pi_{\phi} \right)$ with respect to the norm induced by the inner product $\langle \cdot\, , \, \cdot \rangle_{2, \eta_{\pi}\times 1}$. This means the error vector $Q_{\pi} - \hat{Q}_{\pi}$ will be orthogonal to the subspace ${\rm Span}\left( \Pi_{\phi} \right)$ with respect to  $\langle \cdot\, , \, \cdot \rangle_{2, \eta_{\pi}\times 1}$, implying 
	\[ 
	\langle Q_{\pi}\, , \, \tilde{\pi} \rangle_{\eta_{\pi} \times 1}  = 	\langle \hat{Q}_{\pi}\, , \, \tilde{\pi} \rangle_{\eta_{\pi} \times 1}   \quad \forall  \tilde{\pi} \in {\rm Span}\left(\Pi_{\phi}\right).
	\] 
	Combined with Lemma \ref{lem: pg thm}, this yields the result. 
\end{proof}

Theorem \ref{thm:pg-is-api} is a simple corollary of Lemma \ref{lem: compatible approx}. By Lemma \ref{lem: compatible approx}, Step 2 of Algorithm \ref{alg: fw-pg} is equivalent to $$\tilde{\pi}_{t+1} \in \argmax_{\pi \in \Pi_{\phi}} \, \langle \nabla J(\pi_t) \, , \,  \pi-\pi_t \rangle = \argmax_{\pi \in \Pi_{\phi}} \, \langle \hat{Q}_t \, , \,  \pi-\pi_t \rangle_{\eta_{\pi_t} \times 1},$$
	where $\hat{Q}_t$ has the same definition as in Algorithm \ref{alg: soft-api}.

\subsection{Discussion}
In light of Example \ref{counterexample}, Theorem \ref{thm:pg-is-api} reveals that two changes to API together have an enormous impact. One is to use on-policy state-relevance weights in the choice of loss function that is minimized to select an approximate value function. Lemma \ref{lem: compatible approx} shows that using this estimation loss orients estimation toward accurate evaluation of the decision objective, at least locally. A poor choice of state-weighting seems to have contributed to the poor performance of Algorithm \ref{alg: api} in Example \ref{counterexample}. The ``top states'' $(m+1, \cdots, 2m)$, displayed shaded in the top half of Figure \ref{fig: counterexample}, were given the same weight as the ``bottom states.'' With on-policy state-relevance weights, little weight is given to the top states, as they are visited infrequently. This would allow for an accurate representation at the more important bottom states, dampening the propagating errors from state-aggregation that were shown in the proof of Theorem \ref{thm:api_lower}. \cite{van2006performance} previously observed that the robustness of policies derived from solutions of state-aggregated Bellman equations can depend critically on the choice of state-relevance weights. It is an open question whether his theory can be connected formally to the analysis in this paper.
	
The other change to API is to use soft, or local, changes to the policy. In the $t^{\rm th}$ iteration of Algorithm \ref{alg: soft-api}, the state-relevance weights $\eta_{\pi_t}$ capture relevance under the policy $\pi_t$ by design.  But if the stepsize is large, they may no longer reflect the relevance of states under the policy $\pi_{t+1}$ over which the algorithm is optimizing. Equation \ref{eq:pg_thm_dist_shift} shows that the second order error term, $J(\pi_{t+1})-J(\pi_t) - \langle \nabla J(\pi_t)\,, \, \pi_{t+1}-\pi_t\rangle$, depends on the magnitude of distribution shift, $\eta_{\pi_{t+1}} - \eta_{\pi_t}$. Example \ref{counterexample2} shows that this issue can severely impact decision quality. In that example, the relevance of ``top states'' changes substantially across iterations, and this causes Algorithm \ref{alg: api-adaptive} to cycle endlessly between policies with poor performance. Other works in the literature focus on ensuring policy changes are small enough that performance improves strictly in each iteration \citep{kakade2002approximately,schulman2015trust}, but I am not aware of any results that show such severe degradation in performance is possible otherwise.

\section{Conclusion}
  A surge of recent papers on the theory of reinforcement learning have established convergence rates and sample complexity bounds for different algorithms. Few, however, have elucidated the subtle impact of algorithmic design choices on robustness to approximation errors. This paper has provided one such case study, focused on the comparison between approximate policy iteration and policy gradient when applied with state-aggregated representations. The main contribution is providing a short self-contained treatment with a transparent gap between provable upper and lower bounds.  
  
  One open question, highlighted in Section \ref{sec:upper_bounds} is whether the notion of approximation error in Definition \ref{def:error} can be relaxed to depend only on the optimal value function, as shown for optimistic Q-learning in \cite{dong2019provably}. For simplicity and brevity, this paper has focused on the quality of stationary points and, at times, on the simple projected policy gradient method. Convergence rates for policy gradient methods are given for example in \cite{agarwal2019optimality, bhandari2019global,shani2020adaptive}, and it seems that similar finite time bounds could be developed here. 
  
  Another open direction is to generalize these result beyond the case of state-aggregation.  
  The critical feature of state-aggregated representations is that they can be adjusted locally without impacting the approximation in other regions of the state space. By contrast, in general linear models can be quite rigid, with local changes influencing the approximation at distant states. This risk seems to drive a dependence of past guarantees for policy gradient methods on certain distribution shift terms which were avoided in Theorem \ref{thm: main result}, like the $\kappa_{\rho}$ term in Theorem $\ref{thm:old_result}$. Indirectly, poor dependence on the problem's time horizon was avoided for the same reason\footnote{Past work, like \cite{bhandari2019global,agarwal2019optimality} uses the initial distribution $\rho$ to control likelihood ratio terms as $\eta_{\pi^*}(s)/\eta_{\pi}(s) \geq (1-\gamma)\left(\eta_{\pi^*}(s)/\rho(s)\right)$. Avoiding a dependence of limiting approximation error on distribution shift terms avoided an extra dependence on the effective time horizon, $(1-\gamma)^{-1}$. }.  The deep neural representations that are popular today seem to have elements of both state-aggregation and global linear approximation. A crisp understanding of local methods like state-aggregation may provide some useful intuition for the study of neural networks.






{\small
	\setlength{\bibsep}{2pt plus 0.4ex}
	\bibliography{references}
	\bibliographystyle{abbrvnat}
}

\appendix

\section{Implementing projected policy gradient with aggregated state approximations}

Conceptually, the simplest policy gradient method is the projected gradient ascent iteration
\begin{align*}
	\pi^{t+1} &= {\rm Proj}_{2, \Pi_{\phi} }\left(\pi^{t} + \alpha \nabla J(\pi^{t})\right) = \argmax_{\pi \in \Pi_{\phi}} \left(\pi^{t} + \langle \nabla J(\pi^{t})\, , \,  \pi-\pi^t  \rangle  - \frac{1}{2\alpha} \left\| \pi-\pi^t \right\|^2_2  \right) \quad t\in \mathbb{N},
\end{align*}
where any policy $\pi \in \Pi$ is viewed as a stacked $|\Sc|\cdot |\Ac|$ dimensional vector satisfying $\sum_{a\in \Ac} \pi(s,a) =1$ and $\pi \geq 0$. The operator ${\rm Proj}_{2, \Pi_{\phi} }(\pi) = \argmin_{\pi' \in \Pi_{\phi}} \|\pi' - \pi\|_{2}^2$ denotes orthogonal projection onto the convex set $\Pi_{\phi}$ with respect to the Euclidean norm. The second equality is a well known ``proximal'' interpretation of the projected update \citep{beck2017first}. Although the optimization problem 
\[
\argmax_{\pi \in \Pi_{\phi}} \left(\pi^{t} + \langle \nabla J(\pi^{t})\, , \,  \pi-\pi^t  \rangle  - \frac{1}{2\alpha} \left\| \pi-\pi^t \right\|^2_2  \right)
\]
appears to involve $|\Sc| \cdot |\Ac|$ decision variables, it is equivalent to one involving $m \cdot |\Ac|$ decision variables. 

Algorithm \ref{alg:projected_pg} below uses $\theta \in \mathbb{R}^{m \times |\Ac|}$ to denote a the parameter of a state-aggregated policy, where $\pi_{\infty}(s, a)= \theta_{i,a}$ is the probability of selecting action $a$ for a state $s\in \phi^{-1}(i)$ in segment $i$.   The projection has a simple solution, involving projecting the vector $\tilde{\theta}_{s,:}$ corresponding to  partition  $i$ onto the space of action distributions $\Delta(\Ac)$. Projection onto the simplex can be executed with a simple soft thresholding procedure. In particular, the projection $\hat{y}\in \Delta(\Ac)$ of a vector $y\in \mathbb{R}^{|\Ac|}$ satisfies $\hat{y}_i = \max\{y_i -\beta, 0  \}$ where $\beta$ is chosen so that $\sum_i \hat{y}_i=1$. The pseudocode in Algorithm \ref{algo:projection}, taken from  \citep{duchi2008efficient}, shows how $\beta$ can be found efficiently. The algorithm runs in $O(|\Ac|\log( |\Ac|))$ time, with the bottleneck being the sorting of the vector $y$. \cite{duchi2008efficient} shows how this can be reduced to an $O( |\Ac|)$ runtime. 

\begin{algorithm}[H]\label{alg:projected_pg}
	\caption{Projected Policy Gradient}
	\SetNlSty{texttt}{(}{)}
	\SetAlgoLined
	\SetKwInOut{Input}{input}\SetKwInOut{Output}{output}
	\Input{$\theta\in \reals^{m\times |\Ac|}$, stepsize $\alpha$}
	\BlankLine
	\For{$t=1,2, \cdots, $}{ 
		Get gradient $g = \nabla_{\theta} J(\pi_{\theta})$\;
		Form target $\tilde{\theta} = \theta + \alpha \theta$\; 
		\tcc{Project onto simplex}
		\For{$i=1, \cdots, m$}{ 
			$
			\theta_{i,:} \leftarrow \min_{d\in \Delta(\Ac) } \| d- \tilde{\theta}_{i,:}\|_{2}^2
			$
		}
	}
\end{algorithm}\DecMargin{1em}
\begin{algorithm}[H]
	\caption{Projection onto the probability simplex}\label{algo:projection} 
	\SetNlSty{texttt}{(}{)}
	\SetAlgoLined
	\SetKwInOut{Input}{input}\SetKwInOut{Output}{output}
	\Input{Vector $y\in \mathbb{R}^k$}
	\BlankLine
		Sort $y$ into $\mu$ with $\mu_1 \geq \mu_2 \geq \cdots \geq \mu_k$\; 
		Find  $J= \max\left\{j \in [k] : \mu_j - \frac{1}{j}\left( \sum_{r=1}^{j} \mu_r  -1 \right)   \right\} >0$ \;
		Define $\beta = \frac{1}{J}\left( \sum_{i=1}^{J}  \mu_i -1 \right)$ \;
		\Output{$\hat{y}\in \mathbb{R}^k$ where $\hat{y}_i = \max\{y_i - \beta, 0  \}$ }
\end{algorithm}\DecMargin{1em}

Algorithm \ref{alg: stochastic gradient} provides an unbiased monte-carlo policy gradient estimator. It is based on the formula 
\[
\frac{\partial J(\pi)}{\partial \pi(s,a)}  =   Q_{\pi}(s,a) \eta_{\pi}(s),
\]
which can be derived from the standard policy gradient theorem by picking a direct policy parameterization \citep[see e.g.][]{agarwal2019optimality}.  Rewriting this, if $\tilde{s}_0 \sim \eta_{\pi}$ and $\tilde{a}_0| \tilde{s_0} \sim {\rm Uniform}(1, \cdots,k)$, then 
\[ 
\frac{\partial J(\pi)}{\partial \pi(s,a)}  = Q_{\pi}(s,a) \Prob(\tilde{s}_0=s ) = |\Ac| Q_{\pi}(s,a) \Prob(\tilde{s}_0=s, \tilde{a}_0=a ). 
\]
Using the chain rule, we have 
\[ 
\frac{\partial J(\pi_{\theta})}{\partial \theta_{i,a}} = \sum_{s\in \phi^{-1}(i)} \frac{\partial  J(\pi) }{\partial \pi(s,a)} =  |\Ac| \E\left[ Q_{\pi}(\tilde{s}_0,a) \ind(\tilde{s}_0 \in \phi^{-1}(i), \tilde{a}_0=a ) \right]. 
\]

Algorithm \ref{algo:stochastic_grad} shows how this formula can be used together with a simulator of the environment to generate stochastic gradient $\hat{g}$ with $\E\left[ \hat{g} \right] = \nabla_{\theta} J(\pi_\theta)$. This could be used in stochastic gradient schemes or, by averaging across many independent simulation runs, used to estimate $\nabla_{\theta} J(\pi_{\theta})$ accurately.  The algorithm begins by drawing a state $\tilde{s}_0$ from $\eta_{\pi}(\cdot)$ and then an action $\tilde{a}_0$ uniformly at random (See Remark \ref{rem:uar} ). Then it uses a Monte Carlo rollout to estimate $Q_{\pi}(\tilde{s}_0,\tilde{a}_0)$. To give unbiased estimates of infinite horizon discounted sums underlying $\eta_{\pi}$ and $Q_{\pi}$, it leverages a well know equivalence between geometric discounting and the use a random geometric horizon. For any scalar random variables $\{X_{t}\}_{t=0,1,\cdots}$, one has 
\[ 
\E\left[ \sum_{t=0}^{\infty} \gamma^t X_t \right] = \E\left[ \sum_{t=0}^{\tau} X_t \right] 
\]
where $\tau\sim {\rm Geometric}(1-\gamma)$ has distribution  is independent of $\{X_t\}$. The equivalence is due to the fact that $\Prob(\tau \geq t)=\gamma^t$. 

\begin{algorithm}[H]\label{algo:stochastic_grad}
	\SetNlSty{texttt}{(}{)}
	\SetAlgoLined
	\SetKwInOut{Input}{input}\SetKwInOut{Output}{output}
	\Input{$H$, $S$, $A$, tuning parameters $\{\beta_{k}\}_{k\in \mathbb{N}}$}
	\BlankLine
	\tcc{Sample $\tilde{s}_0 \sim \eta_{\pi}$}
	Sample $\tau \sim {\rm Geometric}(1-\gamma)$\;
	Sample initial $s_0\sim \rho$. \;
	Apply policy  $\pi$ for $\tau$ timesteps\;
	Observe $(s_0, a_0, r_0, \cdots , a_{\tau-1}, r_{\tau-1}, s_{\tau})$\;
	Set $\tilde{s}_0 = s_{\tau}$\;
	\tcc{Draw uniform random action}
	Sample $\tilde{a}_0 \sim {\rm Uniform}\{1,\cdots, |\Ac|\}$\;
	\tcc{Unbiased estimate of $Q_{\pi}(\tilde{a}_0, \tilde{s}_0 ) $}	
	Sample $\tilde{\tau} \sim {\rm Geometric}(1-\gamma)$\;
	Apply action $\tilde{a}_0$ and observe $(\tilde{r}_0, \tilde{s}_1)$ \;
	\If{$\tilde{\tau} >0$ }{
		Apply policy $\pi$ for $\tilde{\tau}$ periods from $\tilde{s}_1$\;
		Observe: $(\tilde{s}_1, \tilde{a}_1, \tilde{s}_2, \cdots, \tilde{a}_{\tau-1}, \tilde{r}_{\tau-1}, \tilde{s}_{\tau})$\; 
	}
	Set $\hat{Q}= \tilde{r}_0 + \cdots + \tilde{r}_\tau$\;
	Find state segment $I = \phi^{-1}(\tilde{s}_0)$ \; 
	Set $\hat{g}(i, a) = \begin{cases} |\Ac|\cdot \hat{Q} &\mbox{if } i=I, a=\tilde{a}_0  \\
		0 & \mbox{otherwise }  \end{cases} $\;
	\Output{ $\hat{g} \in \reals^{m\times  |\Ac|}$}
	\caption{Simple Unbiased Gradient \label{alg: stochastic gradient}}
\end{algorithm}\DecMargin{1em}

\begin{remark}\label{rem:uar}
	A more common presentation of unbiased policy gradient estimation uses a kind of inverse propensity estimate where $\tilde{a}_0$ is sampled from the policy being evaluated \citep{williams1992simple}. This can have very large variance if the policy is nearly deterministic. The form presented above ensures the variance of the sampled gradient is uniformly bounded. 
\end{remark}

\section{Some details on Remark \ref{rem:equivalence_of_representations}}\label{sec:policy_value_equivalence}
Equation \ref{eq:state-agg_policies_and_softmax} can be established as follows. If $Q \in \mathcal{Q}_{\phi}$ is state-aggregated, then it is immediate that $\pi$ defined by $\pi(s,a)  = e^{Q(s, a)}/ \sum_{a'\in \Ac} e^{Q(s, a')}$ is state-aggregated. This shows soft-maximization with respect to some state-aggregated value function yields a state-aggregated policy which assigns non-zero probability to each action. Now we need to show that every strictly stochastic state-aggregated policy can be generated this way. Consider any policy $\pi \in \Pi_{\phi}$ with $\pi(s,a)>0$ for all $a\in \Ac$. Picking $Q(s,a)=\log\pi(s,a)$, yields  $\pi(s,a)  = e^{Q(s, a)}/ \sum_{a'\in \Ac} e^{Q(s, a')}$.
	
Results for policies that assign zero probability to some action are attained by taking limits of strictly stochastic policies which approach them. The minimum in \eqref{eq:deterministic_policies} could be replaced with any deterministic tie-breaking mechanism. This is needed because if ties are broken differently at states sharing common segment, the induced policy would not be constant across segments.

\end{document}